\documentclass{article}

\usepackage{microtype}
\usepackage{graphicx}
\usepackage{subfigure}
\usepackage{booktabs}
\usepackage{xcolor}
\usepackage{hyperref}

\usepackage{amsmath}
\usepackage{amsfonts}
\usepackage{natbib}
\usepackage{tabularx}
\usepackage{amsthm}

\usepackage[accepted]{icml2021}
\usepackage{listings}
\icmltitlerunning{Stochastic Aggregation in Graph Neural Networks}

\begin{document}

\twocolumn[
\icmltitle{Stochastic Aggregation in Graph Neural Networks}




\begin{icmlauthorlist}
\icmlauthor{Yuanqing Wang}{mskcc,cornell,cuny}
\icmlauthor{Theofanis Karaletsos}{facebook}
\end{icmlauthorlist}

\icmlaffiliation{mskcc}{Computational and Systems Biology Program, Sloan Kettering Institute, Memorial Sloan Kettering Cancer Center, New York, NY 10065}

\icmlaffiliation{cornell}{Physiology, Biophysics, and System Biology Ph.D. Program, New York, NY 10065}

\icmlaffiliation{cuny}{M.F.A. Program in Creative Writing, City College of New York, City University of New York, New York, NY 10031}

\icmlaffiliation{facebook}{Facebook, Inc., Menlo Park, CA 94025}

\icmlcorrespondingauthor{Yuanqing Wang}{yuanqing.wang@choderalab.org}

\icmlcorrespondingauthor{Theofanis Karaletsos}{theokara@fb.com}


\icmlkeywords{Graph Neural Network, Variational Inference}

\vskip 0.3in
]
\printAffiliationsAndNotice{}



\begin{abstract}
Graph neural networks (GNNs) manifest pathologies including over-smoothing and limited discriminating power as a result of suboptimally expressive aggregating mechanisms.
We herein present a unifying framework for \textit{stochastic aggregation} (STAG) in GNNs, where noise is (adaptively) injected into the aggregation process from the neighborhood to form node embeddings. 
We provide theoretical arguments that STAG models, with little overhead, remedy both of the aforementioned problems. 
In addition to fixed-noise models, we also propose probabilistic versions of STAG models and a variational inference framework to learn the noise posterior. 
We conduct illustrative experiments clearly targeting oversmoothing and multiset aggregation limitations. 
Furthermore, STAG enhances general performance of GNNs demonstrated by competitive performance in common citation and molecule graph benchmark datasets.

\end{abstract}

\section{Introduction: Aggregation in Graph Neural Networks and its Limitations}
Graph neural networks (GNNs)---neural models that operate on graphs and form node embeddings from its topological neighborhoods---have shown promises in a wide range of domains including social and physical modeling.~\cite{DBLP:journals/corr/KipfW16, xu2018powerful, gilmer2017neural, hamilton2017inductive, battaglia2018relational}
GNNs follow an iterative scheme where the representations of adjacent nodes are pooled with an aggregation function and transformed by a feed-forward neural network.
Working analogously to a Weisfeiler-Lehman (WL) graph isomorphism test~\cite{weisfeiler1968reduction} on a node level~\cite{xu2018powerful} and resembling a series of Laplacian smoothing on a graph level~\cite{DBLP:journals/corr/KipfW16}, such scheme affords GNNs with the ability to generate node embeddings that are rich up to the local symmetry and clustered based on neighborhoods.

Nonetheless, such aggregation scheme also causes limitations of GNNs. 
Firstly, without proper choices of aggregation functions, GNNs are not always as powerful as WL test.
When pooling from (transformed) neighborhood representations, if the underlying set for the neighborhood multiset (See Definition 1 of \citet{xu2018powerful}) is countable, as has been studied in detail in \citet{xu2018powerful}, although different multiset functions learn different attributes of the neighborhood---$\operatorname{MAX}$ learns distinct elements and $\operatorname{MEAN}$ learns distributions---only $\operatorname{SUM}$ is injective and thus capable of achieving the expressive power of WL test.
When the features are continuous, however, \citet{corso2020principal} states that multiple aggregators are needed.
Secondly, the number of layers, which corresponds to the number of steps in a WL test, controls the \textit{locality} of a GNN model, and therefore only deep GNNs can learn long-range relationships; 
unfortunately, deep GNNs suffer from not only \textit{over-fitting} but also \textit{over-smoothing}, where node representation converge to a stationary point dependent only on its degree but not the initial features.~\cite{DBLP:journals/corr/abs-1801-07606, DBLP:journals/corr/abs-1905-10947}

To alleviate these issues, we replace the deterministic aggregation function with stochastic ones and propose a framework which we call STAG, short for \textit{stochastic aggregation}.
At each round of message-passing, we inject randomness into the system by perturbing the weights of the edges according to some distribution, thereby stochastically reweighing the incoming messages.
When training, the gradient of the loss function w.r.t. the parameters could be estimated without bias using Monte Carlo (MC) estimation;
at inference pass, the predictive posterior distribution is formed by marginalize over the edge weight distribution.
The parameters of the distribution of edge weights could either be treated as hyperparameters or jointly trained in an adaptive way under a variational inference (VI) framework.

We summarize our contributions in this paper as follow:
\begin{itemize}
    \item We propose a stochastic aggregation (STAG) framework for GNNs which generalizes Dropout~\cite{JMLR:v15:srivastava14a, gal2016dropout}, DropEdge~\cite{DBLP:journals/corr/abs-1907-10903}, and Graph DropConnect~\cite{DBLP:journals/corr/abs-1907-10903} and expand it to include the perturbation of edge weights using continuous noise distribution, which empirically displays better performance.
    \item We theoretically prove and experimentally demonstrate that STAG, when used with many classes of noise distributions, alleviates the over-smoothing issue and increase expressiveness.
    \item We propose a variational inference (VI) scheme where the parameters of the distributions of edge weights could be learned. Furthermore, such parameters could depend on the graph structure and node embeddings, thus generalizing across graphs.
\end{itemize}

\begin{figure}
    \centering
    \includegraphics[width=0.35\textwidth]{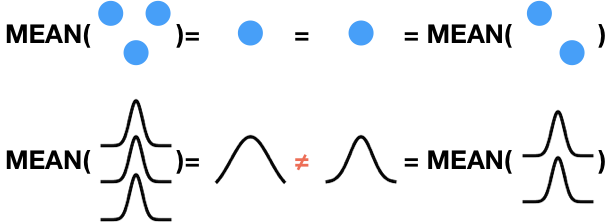}
    \caption{Aggregators display degeneracy issues under the algebra of vectors but not that of random variables.}
    \label{sym}
\end{figure}

\section{Preliminaries}
\subsection{Graph}
A graph is defined as a tuple of collections of nodes and edges $\mathcal{G} = \{\mathcal{V}, \mathcal{E}\}$.
In this paper, we only consider cases where only nodes, but not edges, are attributed; 
node features $[\mathbf{x}_1, \mathbf{x}_2, ..., \mathbf{x}_N] = \mathbf{X} \in \mathbb{R}^{N \times C}$ where $N = \mid \mathcal{V} \mid $ is the number of nodes and $C$ the feature dimension.
Adjacency matrix $\mathbf{A} \in \mathbb{R}^{N \times N}$ associates edges with nodes:
\begin{equation}
    \mathbf{A}_{ij} = \begin{cases}
    1, \: (v_i, v_j) \in \mathcal{E};\\
    0, \: (v_i, v_j) \notin \mathcal{E}.
\end{cases}
\end{equation}

\subsection{Graph Neural Networks}
Modern GNNs could usually be better analyzed through the \textit{spatial} rather than \textit{spectral} lens, according to \citet{DBLP:journals/corr/abs-1901-00596}'s classification.
Following the framework from \citet{xu2018powerful} and \citet{battaglia2018relational}, the $k$-th layer of a GNN could be written as two steps---\textit{neighborhood aggregation}:
\begin{equation}
\label{agg}
a_v^{(k)} = \rho^{(k)}
\big(
h_u^{(k-1)}, u \in \mathcal{N}(v)
\big),
\end{equation}
and \textit{node update}:
\begin{equation}
h_v^{(k)} = \phi^{(k)}(h_v^{(k-1)}, a_v^{(k)}),
\end{equation}
where $h_v^{k}$ is the feature of node $v$ at $k$-th layer, $h_v^{0} = \mathbf{x}_v$ and $\mathcal{N}(\cdot)$ denotes the operation to return the multiset of neighbors of a node. 

Many classical GNNs could be represented in this framework with different choices of \textit{aggregation function} $\rho$ and \textit{update function} $\phi$. 
For instance, Graph Convolutional Network (GCN) by \citet{DBLP:journals/corr/KipfW16}, the graph-level message-passing rule is:
\begin{equation}
\label{kipf-global}
\mathbf{H}^{(k)} = \sigma
\big(
\widetilde{\mathbf{D}}^{-\frac{1}{2}}
\widetilde{\mathbf{A}}
\widetilde{\mathbf{D}}^{-\frac{1}{2}}
\mathbf{H}^{(k-1)}W^{l}
\big)
\end{equation}
could also be analyzed on node-level with $\rho(\cdot)$ being the $\operatorname{MEAN}$ operator and 
\begin{equation}
\label{kipf-local}
\phi(h_v^{(l-1)}, a_v^{(l)}) = \sigma(h_v^{(l)} / \widetilde{\mathbf{D}}_{ii} + a_v^{(l)})
\end{equation}
where $\mathbf{H}^{(l)} \in \mathbb{R}^{N \times C}$ is the node features at a given $l$-th layer, $\widetilde{\mathbf{A}} = \mathbf{A} + \mathbf{I}$ and $\widetilde{\mathbf{D}}$ is a diagonal matrix with $\widetilde{\mathbf{D}}_{ii} = \sum_j \widetilde{\mathbf{A}}_{ij}$.
The equivalence between Equation~\ref{kipf-global} and Equation~\ref{kipf-local} could meanwhile serve as an example to show the equivalence between graph-level- and node-level-view of graph convolution / message-passing procedures.

After the message-passing rounds are finished (and perhaps after post-processing steps consisting of feed-forward layers), the final node representation $\mathbf{H}^{(k)} = [\mathbf{h}^{k}_1, \mathbf{h}^{k}_2, ..., \mathbf{h}^{k}_N]$ could either be connected to a regressor for \textit{node-level} regression or classification, or could be pooled together globally using for example a sum function $h_\mathcal{G} = \sum_{v \in \mathcal{V}} h_v^{k}$ to form the graph representation for \textit{graph-level tasks}.

\subsection{Aggregators}
\label{subsec:agg}
Formally, in the context of multisets and GNNs, the aggregator (or aggregation function) is a function that maps a multiset to the same space of the elements in that multiset.
\begin{equation}
\label{dimension}
\rho: \{\mathbb{R}^{C} \} \rightarrow \mathbb{R}^{C}. \: \text{(dimensionality requirement)}
\end{equation}
Moreover, since there is no notion of ordering in multiset, to qualify for an aggregator, $\rho$ has to be permutation invariant, i.e., for any permutation $P$, 
\begin{equation}
\label{invariance}
\rho(\mathbf{X}) = \rho(P\mathbf{X}). \: \text{(invariance requirment)}
\end{equation}
Practically, under the context of GNN, the multiset input is usually the neighborhood of a node.
Common choices of aggregation function includes:
$\operatorname{SUM}(\mathbf{X}) = \sum_i x_i, \operatorname{MEAN}(\mathbf{X}) = \frac{1}{N} \sum_i x_i$, and
$\operatorname{MAX}(\mathbf{X}) = \max x_i$,
where $\mathbf{X} = \{ x_i, i = 1, 2, ..., N\}.$
More sophisticated architectures, namely attention~\cite{velickovic2018graph}, or Janossy pooling~\cite{DBLP:journals/corr/abs-1811-01900} with arbitrary composing neural function, as long as they satisfy the dimensionality and invariance requirements (Equation~\ref{dimension} and \ref{invariance}), can be used as aggregators.

\section{Related Work}
\subsection{Bayesian Neural Networks}
\label{bnn}
Under the Bayesian formalism, given sets of (input graph, measurement) pairs as training data $\mathcal{D} = \{ \mathcal{G}^{(i)}, y^{(i)}, i=1,2,3,...,n\}$, the probability distribution of the unknown quantity of the measurement $y^*$ which corresponds to the new input graph $\mathcal{G}^*$ could be modelled with respect to the posterior distribution of the neural network parameters $\theta$ as:
\begin{equation}
\label{b}
p(y^{*} | \mathcal{D}, \mathcal{G}^*) = \int p(y^{*} | \mathcal{G}^*, \mathbf{\theta}) p(\mathbf{\theta} | \mathcal{D}) \operatorname{d} \theta.
\end{equation}
This integral, of course, is not tractable and has to be approximated.
The most straightforward way to approximate Equation~\ref{b} would be to sample the \textit{interesting} regions on the weight space and form an ensemble of predictions from Monte Carlo (MC) samples~\cite{neal2012bayesian, mackay1992practical}.
Alternatively, under variational inference (VI) frameworks, we rewrite the posterior distribution of the parameters as a tractable one depending on another set of variational parameters~\cite{Blei_2017, blundell2015weight}.
Dropout~\cite{JMLR:v15:srivastava14a, gal2016dropout} could be regarded as a Bayesian approximation as well.
When the masks of dropout adopts continuous form under Gaussian distribution, whose parameters are jointly optimized, it is equivalent to a variational inference with multiplicative noise~\cite{kingma2015variational}.
Finally, if one uses a delta distribution to model the parameters and searches for the most likely set of neural network parameters under Equation~\ref{b}: $\theta^\text{MAP} = \underset{\theta}{\arg\max} \, p(\theta | \mathcal{D})$, a standard neural network is recovered.

\subsection{Stochastic Regularization for GNNs}
The methods introduced in Section~\ref{bnn} which quantifies uncertainty on weight spaces are all compatible with GNNs.
Additionally, there have been works that introduce stochasticity into GNNs by randommly modifying the structure of the graph: \citet{zhang2018bayesian} regards the input graph as a realization of an underlying graph generated by some random graph generation process;
\citet{DBLP:journals/corr/abs-1801-10247} (FastGCN) randomly removes nodes of input graphs under a Bernoulli distribution;
\citet{DBLP:journals/corr/abs-1907-10903} (DropEdge) randomly removes edges of input graphs; 
\citet{hasanzadeh2020bayesian} (Graph DropConnect) similarly removes edge, although edges are removed independently for each feature dimension.

Under the scheme we propose in this paper, we do not sample on the weight space but rather inject randomness into the aggregation process.
The noise we inject is different at each message passing step whereas in a Bayesian neural network, the weights are kept constant across rounds of message passing.

In Section~\ref{unifiyng}, we show that Dropout, FastGCN, DropEdge, and Graph DropConnect could be viewed as special cases of STAG where the noise distribution in Bernoulli with various dependency structures.
Expanding on these work, we develop a class of methods where we perturb the aggregation process by a \textit{continuous, multiplicative} noise.
In subsequent sections, we theoretically show that STAG with either discrete or continuous noise distributions remedies the over-smoothing tendency as well as limited expressiveness, while STAG with continuous noise display better empirical performance especially when the noise distribution is adaptive.

\section{Theory: Stochastic Aggregation (STAG)}
\label{stag}
At the aggregation stage of graph convolution, STAG samples a set of weights for the edges in the graph under some distribution to come up with effective weighted adjacency matrix,
\begin{equation}
\begin{split}
Z \sim p(Z)\\
\hat{\mathbf{A}} = \mathbf{A} \odot Z
\end{split}
\end{equation}
where $\odot$ denotes Hadamard product.
In this paper we focus on the continous classes of distribution where $p$ takes a Gaussian
\begin{equation}
p(Z) = \mathcal{N}(\mu_Z, \sigma_Z)
\end{equation}
or uniform form
\begin{equation}
p(Z) = \operatorname{Uniform}(a_Z, b_Z).
\end{equation}

We assume that the random mask $Z \in \mathbb{R}^{N \times N}$ has at most $\mid \mathcal{E} \mid$ non-zero elements and are $Z_{ij} = 0$ wherever $\mathbf{A}_{ij} = 0$ (same sparsity).
The weight $Z$ could be either same or different for each layer of message-passing and feature.
If we pack the weight $Z$ across all features and across all message-passing steps, we have a four-dimensional tensor $\mathbf{Z} \in \mathbb{R}^{L, C, N, N}$ where $L$ is the number of steps, $C$ is the number of features (assuming uniform across layers albeit practically it could be different, in which case $\mathbf{Z}$ becomes a \textit{ragged tensor}), and $N$ is the number of nodes.
Such tensor $\mathbf{Z}$ then controls the behavior of the STAG scheme across message-passing rounds.

On a node level, during the $l$-th layer, for each feature channel $i = 1, 2, 3, .., C$, for node $v$, Equation~\ref{agg} becomes
\begin{equation}
\label{theory}
a_v^{(l)}[:, i] = \rho^{(k)}
\big(
(\mathbf{A}_{uv} \cdot \mathbf{Z}[l, i, u, v]) h_u^{(l-1)}[:, i], u \in \mathcal{N}(v)
\big),
\end{equation}
where $\mathcal{N}(\cdot)$ denotes the neighborhood operator of the node.

During inference, with the neural network weights fixed, the joint distribution of the representations at each layer together with the weight tensor can be written as follow:
\begin{multline}
\label{model}
p(\mathbf{H}^{(L)},  \mathbf{H}^{(L-1)}, ..., \mathbf{H}^{(1)}, \mathbf{Z} \mid \mathbf{H}^{(0)}, \mathcal{G}) = \\  \prod\limits_{l=1}^{L} p(\mathbf{Z}[l, :, :, :]) p(\mathbf{H}^{(l)} \mid \mathbf{H} ^{(l-1)}, \mathbf{Z}[l, :, :, :], \mathcal{G}).
\end{multline}
Marginalizing the weights as well as the intermediary representations, we write the marginal distribution of the output of the last layer $\mathbf{H}^{(L)}$ as 
\begin{multline}
\label{posterior}
p(\mathbf{H}^{(L)} \mid \mathbf{H}^{(0)}, \mathcal{G}) = \\ \int \prod\limits_{l=1}^{L} p(\mathbf{Z}) P(\mathbf{H}^{(l)} \mid \mathbf{H} ^{(l-1)}, \mathbf{Z}[l, :, :, :], \mathcal{G}) \operatorname{d}\mathbf{Z}.
\end{multline}

Viewed from a graph level this setting is similar to \citet{zhang2018bayesian} as we take an ensemble of noise-perturbed (same in structure but different weights of edges) graphs as the input for inference.
Compared to \citet{zhang2018bayesian}, our formulation does not need the overhead to conduct convolution for drastically different graphs.
Also, our assumption, more mild and conservative, are reflective of the nature of many classes of graphs---in molecules, chemical bonds become shorter and longer as they vibrate; in a society, how strong the friendship between two certain persons are dynamic rather than static.

\section{STAG as a Unifying Framework}
\label{unifiyng}

\begin{figure}
    \centering
    \includegraphics[width=0.48\textwidth]{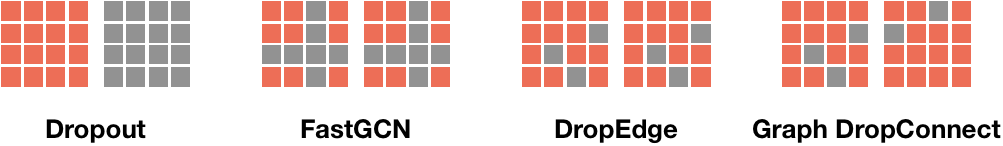}
    \caption{\textbf{Illustration of stochastic regularizing methods on graphs.} Each block denotes a mask on adjacent matrix for each feature. Grey indicates zero.}
    \label{block}
\end{figure}

\textbf{Dropout}~\cite{JMLR:v15:srivastava14a} could be regarded as a case of Equation~\ref{theory} with the first two dimensions being independent and last two dimensions shared;
\begin{equation}
\mathbf{Z}[l, i, :, :] \sim q(Z), Z \in \mathbb{R}^{L \times C};
\end{equation}
$q$ adopts an independent Bernoulli form for binary Dropout and independent normal form for Gaussian Dropout.~\cite{kingma2015variational}

\textbf{FastGCN}~\cite{DBLP:journals/corr/abs-1801-10247} is a case of Equation~\ref{theory} with the first two dimension in $\mathbf{Z}$ sharing samples and 
\begin{equation}
\mathbf{Z}[:, :, v, v] \sim q(Z), Z \in \mathbb{R}^{N},
\end{equation}
with $q$ adopting an independent Bernoulli form.

\textbf{DropEdge}~\cite{DBLP:journals/corr/abs-1907-10903} samples the edges the graph:
\begin{equation}
\mathbf{Z}[:, :, u, v] \sim q(Z), Z \in \mathbb{R}^{\mid \mathcal{E} \mid}
\end{equation}
with independent Bernoulli distribution.

\textbf{Graph DropConnect}~\cite{hasanzadeh2020bayesian} samples edges of graph with Bernoulli distribution independent for each feature:
\begin{equation}
\mathbf{Z}[l, c, u, v] \sim q(Z), Z \in \mathbb{R}^{L \times C \times \mid \mathcal{E} \mid}.
\end{equation}

See Figure~\ref{block} for an illustration of these regularization methods.
More generally, elements of $\mathbf{Z}$ could adopt arbitrary distributions with arbitrary dependency structures.
For instance, they could be completely independent among each other or they could be dependent on variational parameters per-layer, per-graph, per-node, or per-edge, which could in turn be learned from another neural architecture.

\section{STAG Increases Expressiveness}
\label{sec-expressive}

\begin{figure}[htbp]
    \centering
    \includegraphics[width=0.3\textwidth]{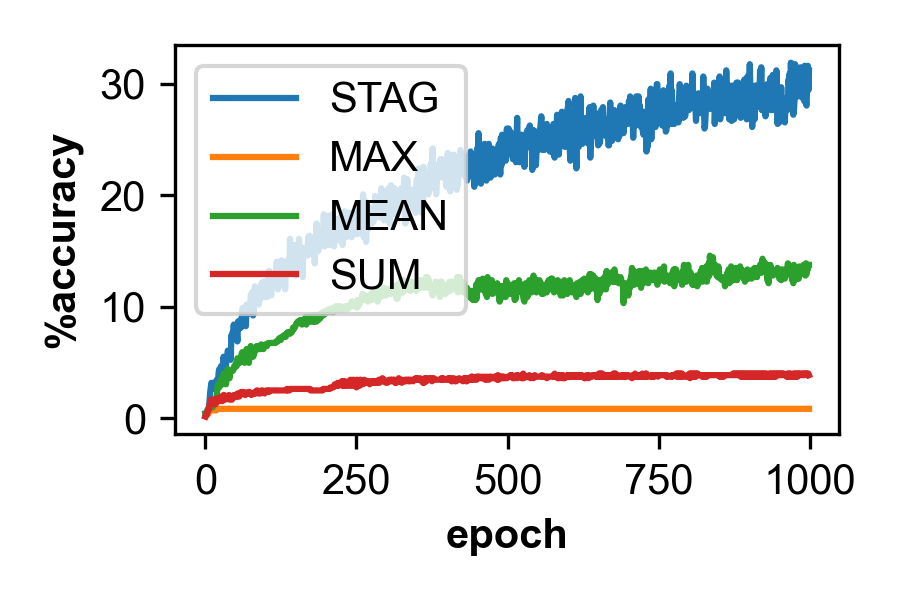}
    \caption{\textbf{STAG allows multisets to be distinguished by feedforward neural networks. } A feed-forward neural network is trained to classify the identity of multisets with underlying set $\{-4, -2, -1, 1, 2, 4\}$ and maximum multiplicity $4$, aggregated using deterministic and stochastic aggregators.}
    \label{fig:multiset}
\end{figure}

The aggregator $\rho$ in Equation~\ref{agg} plays a crucial role in GNNs as it allows neighborhood information to be summarized to form node embeddings, thereby allowing GNNs to approximate Laplacian smoothing on a graph level and WL-test on a node level.
The expressiveness of aggregators has been studied in \citet{xu2018powerful} and \citet{corso2020principal}, for countable and continuous features, respectively.
Particularly, \citet{xu2018powerful} has shown that, among the common aggregation functions, only $\operatorname{SUM}$ is injective if used with \textit{deep multisets} whilst $\operatorname{MEAN}$ and $\operatorname{MAX}$ qualitatively display some desirable merits namely capturing the distribution of elements or distinctive elements.
On the other hand, according to \citet{corso2020principal}, even on $\mathbb{R}$, no aggregator by itself is injective if the support of the multiset is uncountable:
\newtheorem*{number_of_aggregators}{Theorem 1 from \citet{corso2020principal}}
\begin{number_of_aggregators}
In order to discriminate between multisets of size $N$ whose underlying set is $\mathbb{R}$, at least $N$ aggregators are needed.
\end{number_of_aggregators}
Conceptually, such limitation in expressiveness could be seen as a result of the degeneracy under the algebra of vectors (See Figure~\ref{sym}).
For example, for any $x \in \mathbb{R}^C$, we always have $x = \operatorname{MEAN}(\{x, x, x\}) = \operatorname{MEAN}(\{x, x\}) = \operatorname{MAX}(\{x, x, x\}) = \operatorname{MAX}(\{x, x\})$, where the brackets $\{\cdot\}$ denotes multisets.
Nonetheless, random variables on $\mathbb{R}^C$ do not always follow the same algebra (See Figure~\ref{sym}).

As such, when the aggregation process is stochastic, the limitation in expressiveness could be overcame.
To formalize this finding, we now treat the stochastic aggregation process as a basic, deterministic aggregation ($\operatorname{SUM}$, $\operatorname{MEAN}$, $\operatorname{MAX}$ or other discussed in Section~\ref{subsec:agg}) over a \textit{perturbed multiset}, with the following definition:
\newtheorem{definition}{Definition}
\begin{definition}

Suppose $X$ is a multiset with support $\mathbb{R}^C$. 
A perturbation of multiset $X$ using noise distribution $q$ on the same space is defined as:

\begin{equation}
\xi_q(X) = \{z_i \odot x_i, z_i \sim q, x_i \in X\}.
\end{equation}
\end{definition}
The resulting perturbed multiset is a multiset of random variables.
We now prove that only one stochastic aggregator is needed to discriminate between multisets by proving that a \textit{deterministic} aggregator can discriminate between \textit{perturbed multisets.}

\newtheorem{theorem}{Theorem}
\begin{theorem}
\label{one}
Only one aggregator $\rho$ is needed to discriminate between multisets $X$ with support $\mathbb{R}^C \setminus \{ \mathbf{0} \}$ after perturbation with some noise distribution $q$ on $\mathbb{R}^C$.
More formally, under some distribution $q$, $\rho(\xi_q(\mathbf{X}))$ and $\rho(\xi_q(\mathbf{Y}))$ are equal in distribution iff. there exist a permutation $P$ s.t. $[P\mathbf{X}]_i = [\mathbf{Y}_q]_i, \forall 1 \leq i \leq \mid \mathbf{X} \mid$.\footnote{See proof in Section~\ref{supsec-theorems}}
\end{theorem}


Comparing Theorem~\ref{one} and \citet{corso2020principal}, one can think of the perturbation on the multiset as a method to endow the aggregators the ability to pack multiple aggregators into one.
Moreover, note that Theorem~\ref{one} works on \textit{multisets} as opposed to the transformed \textit{deep multisets} as the Lemma 5 in \citet{xu2018powerful} and is therefore more general.
Finally, such gain in expressiveness would not disappear even if one marginalized over the noise distribution, as long as she does so after nonlinearity.
\newtheorem{lemma}{Lemma}
\begin{lemma}
\label{expectation-ok}
There exist some element-wise function $\sigma$ such that
\begin{equation}
\mathbb{E}_{X \sim q_X}(\sigma(X)) \neq \mathbb{E}_{Y \sim q_Y}(\sigma(Y))
\end{equation}
if $X \sim q_X$ and $Y \sim q_Y$ are not equal in distribution.
\end{lemma}
One example of such activation function $\sigma$ is a switch function that equals $1$ on a region where the density of $X$ is strictly greater than $Y$ and $0$ elsewhere.

Combining Theorem~\ref{one} and Lemma~\ref{expectation-ok}, we have that
\begin{lemma}
For some noise distribution $q$, some aggregation function $\rho$, and some element-wise nonlinearity function $\sigma$ the mapping from a multiset with support $\mathbb{R}^C$ to $\mathbb{R}^C$
\begin{equation}
\label{injective}
f(X) = \mathbb{E}_q(\sigma(\rho(\xi_q(X))))
\end{equation}
is injective.
\end{lemma}

Evidently, the injectivity would sustain if the operation in Equation~\ref{injective} is stacked or injective functions (namely some neural networks) are employed between pooling, activation, and marginalization.
Consequently, following Theorem 3 in \citet{xu2018powerful}, a GNN using STAG with appropriate noise distribution and nonlinearity is as powerful as WL-test regardless of the type of basic deterministic aggregators and the countability of the underlying set of features.

We experimentally illustrate the increased expressiveness of STAG in distinguishing multisets in Table~\ref{table-multiset} and Figure~\ref{fig:multiset}.

\begin{table*}[htbp]
    \centering
    \footnotesize
    \resizebox{\textwidth}{!}{%
    \begin{tabular}{c | c c c c c | c}
    \hline
    & $\operatorname{SUM}$
    & $\operatorname{MEAN}$
    & $\operatorname{MAX}$
    & $\operatorname{MIN}$
    & $\operatorname{STD}$
    & $\mathbb{E}(\sigma(\rho(\xi_{q}(\cdot)))) = \mathbb{E}(\exp(\operatorname{SUM}(\xi_{\operatorname{Uniform}(0, 1)}(\cdot))))$ \\
    \hline
    $\{2, 2\}; \{0, 4\}$
    & $4 = 4$
    & $2 = 2$
    & $2 \textcolor{magenta}{\neq} 4$
    & $2 \textcolor{magenta}{\neq} 0$
    & $0 \textcolor{magenta}{\neq} 2$
    & $(e^2-1)^2 \textcolor{magenta}{\neq} e^4 - 1$
    \\
    
    $\{0, 2, 2\}; \{0, 0, 2\}$
    & $4 \textcolor{magenta}{\neq} 2$
    & $\frac{4}{3} \textcolor{magenta}{\neq} \frac{2}{3}$
    & $2 = 2$
    & $0 = 0$
    & $\frac{2\sqrt{2}}{3} = \frac{2\sqrt{2}}{3}$
    & $e^4 - 1 \textcolor{magenta}{\neq} e^2 - 1$ \\
    
    $\{0, 2, 2, 4\}; \{0, 0, 4, 4\}$
    & $8=8$
    & $2=2$
    & $4=4$
    & $0=0$
    & $\sqrt{6} \textcolor{magenta}{\neq} 2\sqrt{2}$
    & $ (-1 + e^2)^3 (1 + e^2) \textcolor{magenta}{\neq} (e^4-1)^2 $ \\
    
    $\{1, 1, 4\}; \{0, 3, 3\}$
    & 6 = 6 
    & 2 = 2
    & $4 \textcolor{magenta}{\neq} 3$
    & $1 \textcolor{magenta}{\neq} 0$
    & $\sqrt{2} = \sqrt{2}$
    &  $ (-1 + e)^3 (1 + e + e^2 + e^3) \textcolor{magenta}{\neq} (e^3-1)^2$ \\
    \hline
    \end{tabular}}
    \caption{\textbf{STAG distinguishes multisets indistinguishable by other aggregators.} Example multisets with underlying set $\mathbb{R}$ are taking from \citet{corso2020principal}. We chose $\rho=\operatorname{SUM}$ as basic aggregator, $q=\operatorname{Uniform}(0, 1)$ as the noise distribution, and $\sigma(\cdot) = \exp(\cdot)$ as activation function. Cases where the aggregator succeed in distinguish the multisets are marked red. It's worth mentioning that we proved Theorem~\ref{one} on the space excluding $\mathbf{0}$ as $\xi_q$ with multiplicative noise $q$ would not be able to count the number of zeros so the aggregated representation for $\{0, 2, 2\}$ and $\{2, 2\}$ are the same. One can circumvent this degeneracy by using an injective mapping onto $\mathbb{R}^+$ before aggregating.}
    \label{table-multiset}
\end{table*}

\section{STAG Alleviates Over-Smoothing}
\label{sec-oversmooth}
\begin{figure}
    \centering
    \includegraphics[width=0.4\textwidth]{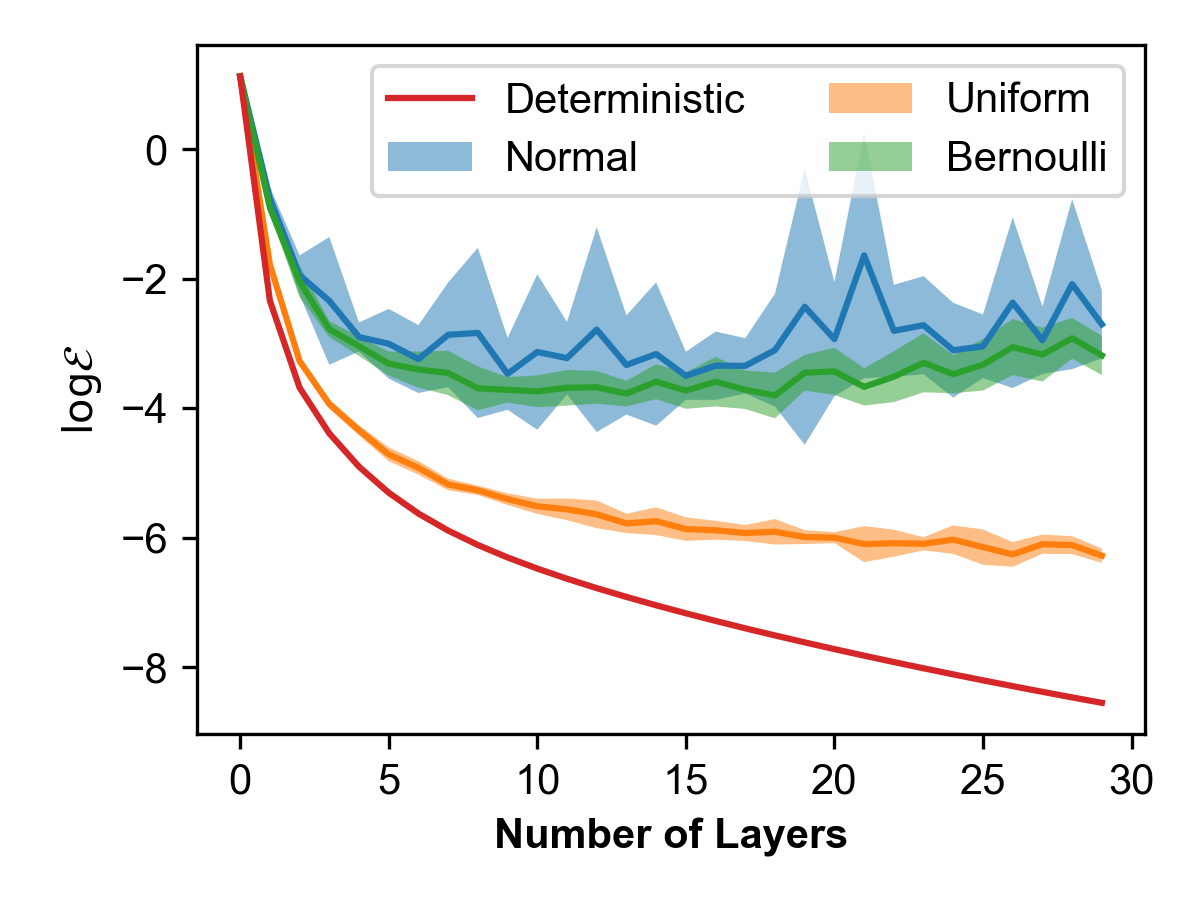}
    \caption{\textbf{STAG slows down the decrease of Dirichlet energy.} A random signal is applied on a random graph and aggregation are conducted multiple times. We perturb the input of aggregation by a multiplicative noise under distributions of various class with 0.5 mean and 0.25 variance and plot the trend of Dirichlet energy.}
    \label{dirichlet-plot}
\end{figure}
As more layers go into a GNN model, not only will it experience \textit{over-fitting} which is ubiquitous in all neural models when over-parametrized, but it will also have the tendency of \textit{over-smoothing}---a behavior studied in \citet{DBLP:journals/corr/abs-1801-07606, DBLP:journals/corr/abs-1905-10947} where node representation converge to a subspace dependent only on topology but not the initial features as a result of repetitive Laplacian smoothing.
Worse still, by the equivalence between WL-test and GNN layers~\cite{xu2018powerful}, only deep GNN architectures can capture longer-range dependencies, thus presenting a dilemma.
It has been studied in \citet{DBLP:journals/corr/abs-1907-10903} that dropping sufficient edges in a graph would make its second-smallest eigenvalue of Laplacian smaller, until it approaches zero (disconnected graph), delaying the smoothing process.
Here, to study whether and how STAG alleviates the over-smoothing tendency of GNNs, we adopt \citet{cai2020note}'s framework and focus on the Dirchilet energy of a signal on a graph:
\newtheorem{note-definition}{Definition 3.1 from \citet{cai2020note}}
\begin{note-definition}
Dirichlet energy $\mathcal{E}(f)$ of scalar function $f$ on the graph G is defined as
\begin{equation}
\mathcal{E}(f) = f^{T} \widetilde{\Delta}f
= \frac{1}{2} \sum A_{ij}(\frac{f_i}{\sqrt{1+d_i}} - \frac{f_j}{\sqrt{1+d_j}})^2,
\end{equation}
where $\widetilde{\Delta}$ is the normalized Laplacian $\widetilde{\Delta} = \mathbf{I} - \widetilde{D}^{-\frac{1}{2}}\widetilde{A}\widetilde{D}^{-\frac{1}{2}}$ and $d_i = D_{ii}$.
For a vector field $\mathbf{X} \in \mathbb{R}^{N \times C}$, Dirichlet energy is defined as
\begin{equation}
\mathcal{E}(\mathbf{X}) = \operatorname{tr}(\mathbf{X}^T \widetilde{\Delta} \mathbf{X}).
\end{equation}
\end{note-definition}
Now, using $\rho (\mathbf{X})$ to denote the simultaneous application of some neighrbohood aggregation function $\rho$ on node features $\mathbf{X}$, we state that
\begin{theorem}
\label{dirichlet-slow}
For any multiplicative noise distribution $q$ satisfying $\mid \mathbb{E}_{z \sim q}(z) \mid \geq 1$, any deterministic aggregator $\rho$, a node representation $\mathbf{X}$ of a graph, we have:
\begin{equation}
\mathbb{E}_{q}(\mathcal{E}(\rho(\xi_q(\mathbf{X})))) \geq \mathcal{E}(\rho (\mathbf{X}))
\end{equation}
\end{theorem}
In other words, the graph convolution with aggregation input perturbed by such distribution $q$ is expected to be less smooth and converge to the subspace independent of the initial features of graphs slower.
The condition $\mid \mathbb{E}_{z \sim q}(z) \mid \geq 1$ is sufficient but not necessary.
This bound also correspond to the finding in \citet{DBLP:journals/corr/abs-1905-10947} that increasing the scale of the neural network weights alleviates over-smoothing and enhances GNN performance.
We also experimentally illustrate Theorem~\ref{dirichlet-slow} in Figure~\ref{dirichlet-plot} and apply this on the benchmark test from \citet{DBLP:journals/corr/KipfW16} in Figure~\ref{depth}.
\begin{figure}
    \centering
    \includegraphics[width=0.45\textwidth]{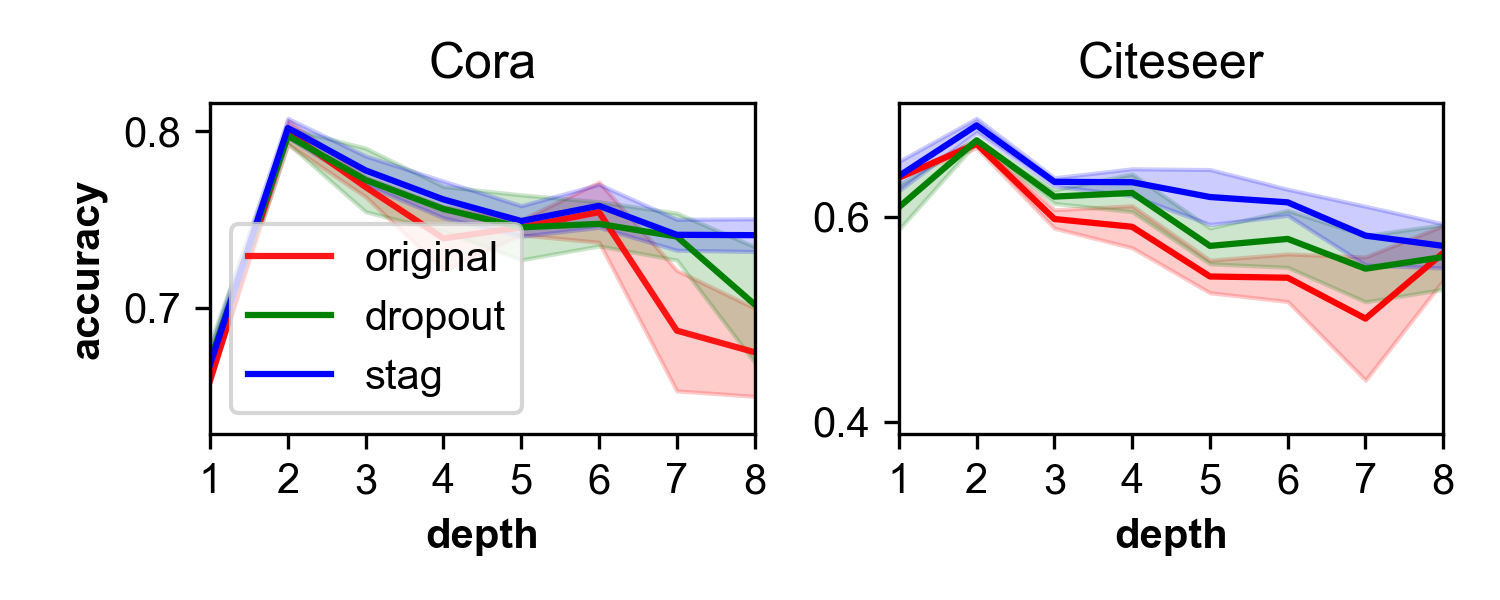}
    \caption{\textbf{STAG alleviates performance deterioration for deep GNN.} We adopt the experiment setting in \citet{DBLP:journals/corr/KipfW16} and plot the test set performance on Cora and Citeseer datasets against number of GCN layers with Dropout and STAG.}
    \label{depth}
\end{figure}

\begin{table*}[htbp]
    \centering
    \scriptsize
    \resizebox{\textwidth}{!}{%
    \begin{tabular}{c c c c c|c c c c }
    \hline
    & \multicolumn{2}{c}{Cora}
    & \multicolumn{2}{c}{Citeseer}
    & \multicolumn{2}{c}{ESOL}
    & \multicolumn{2}{c}{FreeSolv}\\
    \hline
    & 2 layers
    & 4 layers
    & 2 layers
    & 4 layers
    & 2 layers
    & 4 layers
    & 2 layers
    & 4 layers\\
    \hline
    Deterministic
    & 79.34 ± 0.22
    & 77.52 ± 0.33
    & 68.20 ± 0.34
    & 60.66 ± 1.19
    & 0.7003 ± 0.0638
    & 0.6435 ± 0.0550
    & 1.1643 ± 0.1105
    & 1.2230 ± 0.0589
    \\
    \hline
    
    $\operatorname{Normal}(1, 0.2)$
    & 76.36 ± 0.53
    & 76.36 ± 0.53
    & 67.64 ± 0.43
    & 61.72 ± 0.90
    & 0.6329 ± 0.0112
    & 0.6418 ± 0.0253
    & 1.1481 ± 0.0626
    & 1.2354 ± 0.0533
    \\
    
    $\operatorname{Normal}(1, 0.4)$
    & 79.74 ± 0.31
    & 77.68 ± 0.78
    & 67.90 ± 1.00
    & 62.18 ± 0.84
    & \textbf{0.5960 ± 0.0375}
    & \textbf{0.6096 ± 0.0334}
    & 1.1408 ± 0.0710
    & 1.1664 ± 0.0406
    \\
    
    $\operatorname{Normal}(1, 0.8)$
    & \textbf{80.34 ± 0.45}
    & 77.68 ± 0.73
    & 66.92 ± 1.99
    & 62.74 ± 0.62
    & 0.6589 ± 0.0323
    & 0.6240 ± 0.0349
    & 1.1703 ± 0.0767
    & 1.2308 ± 0.0904
    \\
    \hline
    
    $\operatorname{Uniform} (0.8, 1.2)$
    & 79.46 ± 0.31
    & \textbf{79.72 ± 0.37}
    & 67.86 ± 0.52
    & 61.38 ± 0.85
    & 0.6357 ± 0.0241
    & 0.6804 ± 0.0515
    & 1.1799 ± 0.0449
    & \textbf{1.1317 ± 0.0435}
    \\
    
    $\operatorname{Uniform} (0.6, 1.4)$
    & 79.72 ± 0.37
    & 76.58 ± 0.95
    & 67.94 ± 0.63
    & 61.26 ± 1.53
    & 0.6444 ± 0.0525
    & 0.6344 ± 0.0181
    & 1.2313 ± 0.1357
    & 1.2256 ± 0.1111
    \\
    
    $\operatorname{Uniform}(0.2, 1.8)$
    & 79.86 ± 0.34
    & 77.72 ± 0.84
    & 67.60 ± 0.59
    & 61.84 ± 1.06
    & 0.6712 ± 0.0432
    & 0.6478 ± 0.0301
    & 1.1549 ± 0.0664
    & 1.1614 ± 0.0780
    \\
    \hline
    
    $\operatorname{Bernoulli}(0.2)$
    & 80.08 ± 0.38
    & 77.86 ± 1.14
    & 68.06 ± 0.72
    & 62.26 ± 1.76 
    & 0.6488 ± 0.0293
    & 0.6331 ± 0.0280
    & 1.1424 ± 0.0922
    & 1.2301 ± 0.1113
    \\
    
    $\operatorname{Bernoulli}(0.4)$
    & 80.06 ± 0.67
    & 77.26 ± 0.55
    & 67.16 ± 0.41
    & 61.68 ± 0.41
    & 0.6069 ± 0.0340
    & 0.6368 ± 0.0321
    & 1.1732 ± 0.0555
    & 1.1717 ± 0.0749
    \\
    
    $\operatorname{Bernoulli}(0.8)$
    & 15.48 ± 0.58
    & 54.06 ± 3.12
    & 17.74 ± 0.61
    & 18.48 ± 0.82
    & 0.6200 ± 0.0184
    & 0.6290 ± 0.0207
    & \textbf{1.1394 ± 0.0714}
    & 1.1365 ± 0.0841
    \\
    \hline
    
    $\operatorname{DE}(0.2)$
    & 79.86 ± 0.38
    & 76.75 ± 0.70
    & 62.30 ± 1.35
    & 61.56 ± 0.85
    & 0.7381 ± 0.0202
    & 0.7416 ± 0.0248
    & 1.4772 ± 0.0463
    & 1.5224 ± 0.0488
    \\
    
    $\operatorname{DE}(0.4)$
    & 79.50 ± 0.81
    & 76.86 ± 1.02
    & 69.06 ± 0.88
    & \textbf{63.66 ± 1.65}
    & 0.7133 ± 0.0262
    & 0.7200 ± 0.0339
    & 1.5640 ± 0.0282
    & 1.5161 ± 0.0824
    \\
    
    $\operatorname{DE}(0.8)$
    & 71.08 ± 1.44
    & 67.28 ± 0.46
    & 57.52 ± 2.09
    & 46.70 ± 1.04
    & 0.7368 ± 0.0256
    & 0.7336 ± 0.0310
    & 1.5406 ± 0.0895
    & 1.6025 ± 0.0570
    \\
    
    $\operatorname{GDC}(0.2)$
    & 79.74 ± 0.32
    & 77.38 ± 0.73
    & \textbf{69.22 ± 0.87}
    & 61.78 ± 1.27
    & 0.6178 ± 0.0240
    & 0.6133 ± 0.0381
    & 1.2004 ± 0.0344
    & 1.1346 ± 0.0385
    \\
    
    $\operatorname{GDC}(0.4)$
    & 79.66 ± 0.45
    & 77.80 ± 0.68
    & 68.72 ± 0.47
    & 62.46 ± 0.87
    & 0.6400 ± 0.0426
    & 0.6345 ± 0.0200
    & 1.1783 ± 0.0457
    & 1.2135 ± 0.1248
    \\
    
    $\operatorname{GDC}(0.8)$
    & 76.14 ± 0.43
    & 76.18 ± 0.65
    & 60.80 ± 0.64
    & 57.60 ± 1.79
    & 0.6271 ± 0.0235
    & 0.6408 ± 0.0315
    & 1.2102 ± 0.0926
    & 1.1774 ± 0.0483
    \\
    \hline
    \end{tabular}%
}
\caption{\textbf{Performance of STAG on citation and molecule graph datasets.} For Cora and Citeseer, we report the node classification accuracy (higher is better); for ESOL and FreeSolv, we report the graph regression RMSE (log mol per liter and kcal/mol) (lower is better). We report the mean and standard deviation across five runs. DE: DropEdge~\cite{DBLP:journals/corr/abs-1907-10903}; GDC: Graph DropConnect~\cite{hasanzadeh2020bayesian}}
\label{bigtable}
\end{table*}

\section{Variational Inference with STAG}
With non-adaptive STAG, we have insofar been sampling the prior distribution whose parameters are treated as hyperparameters.
In this section, we tune such parameters jointly and adaptively.
Having generalized the noise distribution in STAG to continuous distributions affords us the flexibility to apply variational inference to learn interesting uncertainty structures.

We define a variational family $q(\mathbf{Z})$ over the noise injection variables per layer. For this , when operating under a Normal prior, we utilize a factorized Normal distribution which corresponds to a \textit{mean-field} assumption as the approximate posterior over edge weights $\mathbf{Z}$: $q(\mathbf{Z}) = \prod \limits_{l=1}^{L}q(\mathbf{Z}[l,:])$
with $q(\mathbf{Z}[l,:]) = \mathcal{N}(\mu_{\mathbf{Z}}, \sigma_{\mathbf{Z}})$ with $\phi$ denoting the collection of parameters for the variational family.

If we have a target variable $\mathbf{y}$ and a likelihood model $p(\mathbf{y}|\mathbf{H}^{(L)})$, with fixed weights, we can maximize the data evidence through the \textit{evidence lower bound} (ELBO) given as: 
\begin{multline}
\mathcal{L}(\phi) =\\ \mathbb{E}_{\mathbf{Z} \sim q(\mathbf{Z})} \text{log}\frac{p(\mathbf{y}|\mathbf{H}^L)p(\mathbf{H}^{(L)}, ..., \mathbf{H}^{(1)}, \mathbf{Z}, \mid \mathbf{H}^{(0)}, \mathcal{G})}{q(\mathbf{Z})}.
\end{multline}
A general recipe to construct losses given this would be to descent $-\mathcal{L}(\phi)$.


Now we provide four dependency structures to parametrize $\phi = \{ \mu_{\mathbf{Z}}, \sigma_{\mathbf{Z}} \}$.
In the following paragraphs, we denote the various algorithms by the dimensionality of the variational parameters $\{ \mu_{\mathbf{Z}}, \sigma_{\mathbf{Z}} \}$.
Since samples are always acquired in the space of $\mathbb{R}^{L, C, N, N}$, the rest of the dimensions are sampled independently.

$\textbf{STAG}_\textbf{VI} (\mathbb{R})$: Firstly, can have a simple variational inference (VI) model where $\mu_{\mathbf{Z}} \in \mathbb{R}$ and $\sigma_{\mathbf{Z}} \in \mathbb{R}$ are \textit{gloabl} parameters not dependent upon either the structure of the feature of the graph. 

$\textbf{STAG}_\textbf{VI} (\mathbb{R}^{C})$: Similarly to \citet{hasanzadeh2020bayesian}'s improvement over \citet{DBLP:journals/corr/abs-1907-10903}, we allow each \textit{feature} to learn its own noise, and have $\mu_{\mathbf{Z}} \in \mathbb{R}^{C}$ and $\sigma_{\mathbf{Z}} \in \mathbb{R}^{C}$.

$\textbf{STAG}_\textbf{VI} (\mathbb{R}^{\mid \mathcal{E} \mid})$: We expand our model into a transductive one by utilizing amortized inference over the variational parameters conditioned on the topology of the graph.
Now $\mu_{\mathbf{Z}} \in \mathbb{R}^{\mid \mathcal{E} \mid}$ and $\sigma_{\mathbf{Z}} \in \mathbb{R}^{\mid \mathcal{E} \mid}$ become local variables and are learned from a feedfoward neural network following another graph neural network.
\begin{gather*}
\mathbf{F} = \operatorname{GNN}(\mathcal{G}, \mathbf{X})\\
\mu_{\mathbf{Z}}, \sigma_{\mathbf{Z}} = \operatorname{NN}(\mathbf{F})
\end{gather*}
This would endow the model with generalizability towards unseen graphs.

$\textbf{STAG}_\textbf{VI} (\mathbb{R}^{\mid \mathcal{E} \mid \times C})$: Finally, if we further enrich the model by learning one set of variational parameter for each edge and for each feature, similarly connecting the node representation from an encoding network, we have $\mu_{\mathbf{Z}} \in \mathbb{R}^{\mid \mathcal{E} \mid \times C}$ and $\sigma_{\mathbf{Z}} \in \mathbb{R}^{\mid \mathcal{E} \mid \times C}$

\begin{table}[htbp]
    \centering
    \resizebox{0.45\textwidth}{!}{%
    \begin{tabular}{c c c | c c}
    \hline
        &  Cora & Citeseer & \#Params & Iter. Time \\
    \hline
    $\text{STAG}_\text{VI} (\mathbb{R})$ & 80.08 ± 0.73 & 66.53 ± 0.33 & 184k & 14.3 ms \\
    $\text{STAG}_\text{VI} (\mathbb{R}^{C})$ & 81.33 ± 0.62 & 68.53 ± 0.54 & 188k & 15.5 ms \\
    $\text{STAG}_\text{VI} (\mathbb{R}^{\mid \mathcal{E} \mid})$ & 80.18 ± 0.73 & 66.48 ± 0.53 & 386k & 24.5 ms\\
   $\text{STAG}_\text{VI} (\mathbb{R}^{\mid \mathcal{E} \mid \times C})$ & \textbf{81.38 ± 0.40} & \textbf{71.28 ± 0.65} & 1186k & 30.0 ms\\
    \hline
    BBGDC & 81.32 ± 0.53 & 70.96 ± 0.72 & 475k & 23.5 ms\\
    $\text{STAG}_\text{MLE}$ (best) & 80.34 ± 0.45 & 69.22 ± 0.87 & 184k & 9.3 ms\\
    \hline 
    \end{tabular}}
    \caption{Performance of STAG with variational inference (VI) on citation graph datasets}
    \label{citation-performance}
\end{table}

\begin{table}[htbp]
    \centering
    \resizebox{0.35\textwidth}{!}{%
    \begin{tabular}{c c c}
    \hline
    & ESOL & FreeSolv\\
    \hline
    $\text{STAG}_\text{VI} (\mathbb{R})$ &0.5956 ± 0.0200 & 1.1500 ± 0.0359 \\
    
    $\text{STAG}_\text{VI} (\mathbb{R}^C) $ & 0.6221 ± 0.0142 & 1.1561 ± 0.0803\\
    
    $\text{STAG}_\text{VI} (\mathbb{R}^{\mid \mathcal{E} \mid}) $ & 0.6901 ± 0.0427 & 1.3349 ± 0.1513\\
    
    $\text{STAG}_\text{VI} (\mathbb{R}^{\mid \mathcal{E} \mid \times C})$ & \textbf{0.5928 ± 0.0326} & \textbf{0.9958 ± 0.0768} \\
    \hline
    $\text{STAG}_\text{MLE}$ (best) & 0.5960 ± 0.0375 & 1.1394 ± 0.0714\\
    \hline
    \end{tabular}}
    \caption{Performance of STAG with variational inference (VI) on molecule graph datasets}
    \label{mol-performance}
\end{table}

We experimentally show the performance of $\text{STAG}_\text{VI}$ in Section~\ref{sec-expressive}, Table~\ref{citation-performance}, and Table~\ref{mol-performance}.
In Section~\ref{subsec-bbb}, we also compare $\text{STAG}_\text{VI}$ with a VI method Bayes-by-Backprop~\cite{blundell2015weight} that quantifies the \textit{weight uncertainty} rather than \textit{structural uncertainty}.

\section{Experiments}
\label{sec:exp}
\subsection{Illustrative Experiments}
\textbf{For Section~\ref{sec-expressive}: STAG Increases Expressiveness}:

To show the superior expressiveness of STAG, we adopt the example from \citet{corso2020principal} and show in Table~\ref{table-multiset} that the aggregator in Equation~\ref{injective} can succeed in distinguishing all the toy set which couldn't be distinguished by other aggregators.

Inspired by this example, we also perform a toy classification task where a feed-forward neural network of two layers with 128 units each is trained to distinguish multisets with underlying set $\{-4, -2, -1, 1, 2, 4\}$ with multiplicity up to $4$, aggregated by $\operatorname{SUM}$, $\operatorname{MAX}$, $\operatorname{MEAN}$ aggregators as well as stochastic aggregator $\mathbb{E}_q(\sigma(\rho(\xi_q(\cdot))))$ with $q$ being $\operatorname{Uniform}(0, 1)$ and $\rho$ being $\operatorname{SUM}$.
This mimics the aggregation-neural transformation process in graph neural networks.
We plot the training curve in Figure~\ref{fig:multiset}.

\textbf{For Section~\ref{sec-oversmooth}: STAG Alleviates Over-Smoothing}: 
Following the experimental setting in \citet{cai2020note}, we generate a random geometric graph with 200 nodes and radius 0.125. 
A input signal is generated by linearly combining the eigenvectors corresponding to the first 20 eigenvalues of the graph.
In each layer, we set the embedding of the node to be the average of its neighborhood with self-loop and normalization, which corresponds to $P = I - \widetilde{\Delta}$ and the message-passing step in \citet{DBLP:journals/corr/KipfW16}.
The aggregation is either deterministic or perturbed distributions of some class with mean 0.5 and variance 0.25.
Plotting the mean and standard deviation of Dirichlet energy across ten runs against number of layers of graph convolution in Figure~\ref{dirichlet-plot}, we notice that normal, uniform, and Bernoulli (which corresponds to DropEdge~\cite{DBLP:journals/corr/abs-1907-10903}) noise distribution all slow the decrease of Dirichlet energy.

To show that delaying the over-smoothing effect of GNNs also boost the performance on real-world datasets, we followed the protocols in \citet{DBLP:journals/corr/KipfW16} and trained Graph Convolutional Networks (GCN) with 16 units, ReLU activation function, and from two to eight number of layers.
Adam optimizer~\cite{kingma2017adam} with learning rate 0.01 are used for these experiments with $5 * 10^{-4}$ L2 regularization on the first layer.
The dropout probability is chosen to be 0.5 which is the same as the original paper.
The noise distribution for STAG is randomly set to be $\mathcal{N}(1, 1)$.
We plot the mean and standard deviation of the test set accuracy against the number of layers in Figure~\ref{depth}.

\subsection{None-Adaptive STAG}
We empirically show the benefits in Section~\ref{sec-expressive} and Section~\ref{sec-oversmooth} using node classification tasks on citation networks (Cora and Citeseer) and graph regression tasks on molecular graph (ESOL~\cite{delaney2004esol} and FreeSolv~\cite{mobley2014freesolv}) datasets.
We used Graph Convolutional Network (GCN)~\cite{DBLP:journals/corr/KipfW16} for all of our experiments, although STAG is compatible with almost all variants of GNNs (See Section~\ref{supsec-var}).
ReLU activation function is used everywhere.
We used the same training/validation/test split as in \citet{DBLP:journals/corr/KipfW16}: 140 training nodes, 500 validation nodes, and 1000 testing nodes for Cora and 120 training nodes, 500 validation nodes, and 1000 testing nodes for Citeseer.
For ESOL and FreeSolv, we randomly split training/validation/test with a 80-20-20 proportion with fixed random seed.
Using a similar experimental setting in \citet{hasanzadeh2020bayesian}, we report the performance of two- and four-layer graph convolutional network (GCN)\cite{DBLP:journals/corr/KipfW16} with 128 units each layer and ReLU activation function.
All models are trained for 2000 epochs with early stopping with Adam~\cite{kingma2017adam} optimizer with $5 * 10^{-3}$ learning rate a L2 regularization of $5 * 10^{-3}$ on the input layer.
Five runs are conducted for each experiment and we report the mean and standard deviation.

As shown in Table~\ref{bigtable}, STAG with various noise distributions almost constantly outperforms the deterministic baseline.
It is worth mentioning that the only difference between a STAG with a Bernoulli distribution and the Graph DropConnect~\cite{hasanzadeh2020bayesian} is that Graph DropConnect normalizes the adjacency matrix to have its original in-degree after dropping edges as it has been argued in \citet{hasanzadeh2020bayesian} that normalization remedies vanishing gradient.
When it comes to continuous distribution centered on $1$, however, the effect of normalizing operation is minimal and we empirically observed longer training time and worsened performance if normalizing operations are used for STAG with continuous noise distribution.

\subsection{STAG with Variational Inference}
We tested the $\text{STAG}_\text{VI}$ models on the same datasets: Cora, Citeseer, ESOL~\cite{delaney2004esol}, and FreeSolv~\cite{mobley2014freesolv}.
Since we observed that two-layer GNNs generally outperform four-layer ones, we used two-layer GCN~\cite{DBLP:journals/corr/KipfW16} throughout the experiments.
The rest of the experiment setting are identical to non-adaptive version of STAG, with the exception that we used a $10^{-3}$ learning rate for all models.
All of the feedforward neural networks to determine the variational parameters consist of two layers connected with ReLU activation function.
Using validation sets, we tuned the initializing values of $\mu$ and $\sigma$ parameters as well as the parameters in the prior on edge weights for each task.
The hyperparameters and settings for Graph DropConnect~\cite{hasanzadeh2020bayesian} is adopted from its original publication.

We notice that $\text{STAG}_\text{VI}$ constantly outperform the best of non-adaptive (or maximum-likelihood estimate, MLE) counterparts. 
The most sophisticated model, $\text{STAG}_\text{VI} (\mathbb{R}^{\mid \mathcal{E} \mid \times C})$, where the variational parameters are learned for every edge and every feature, consistently achieve the best results among the models.

With the competitive performance on small molecule datasets, we show that $\text{STAG}_\text{VI}$ can generalize across graphs.\footnote{For more experimental details, see Section~\ref{supsec-detail} and \url{https://github.com/yuanqing-wang/stag.git}}.

\section{STAG is Lightweight}
\textbf{Engineering Complexity: }The non-adaptive version of STAG could be implemented under the framework of Deep Graph Library (DGL)~\cite{wang2020deep} and PyTorch~\cite{NEURIPS2019_9015} in one line:
\begin{verbatim}
dgl.function.copy_src = lambda edges: {
    'm': edges.src['h'] 
        * Normal(1, 1).sample(
            edges.src['h'].shape
        )
}
\end{verbatim}

\textbf{Runtime Complexity: } For sparse adjacency matrix, the runtime complexity for sampling the weights on edges is $\mathcal{O}(\mid \mathcal{E} \mid)$, which is comparable with the graph convolution itself.
Using the one-line implementation in the previous section, we also benchmarked the running speed of our model with two layers on Cora dataset with 128 units on a Tesla V100 GPU and found a 5.9 to 9.3 ms increase in iteration time.
The iteration time on V100 GPUs of variational inference models are included in Table~\ref{citation-performance}.

\section{Discussion}
In this paper we proposed a unifying framework that injects stochasticity into the GNN systems by sampling the edge weights at each message-passing step.
Our framework increases the expressiveness of GNNs and alleviate their over-smoothing tendencies, as proved by theoretical analysis and evidenced by illustrative and benchmarking experiments.
We also develop a variational inference version of STAG where the parameters of the noise distribution is jointly tuned with the model parameters, which showed even further improvement in benchmark tests.

For $\text{STAG}_\text{VI}$, with more interesting dependency structures in the noise distributions, we would like to study whether GNNs with STAG would be able to surpass the expressiveness of WL-test, especially when used with higher-level variants of GNNs.
For non-adptive versions of STAG, we plan to further optimize the sampling efficiency of STAG models, in order to make STAG a simple and ultra-lightweight trick to boost the performance of GNNs.
We hope this work would encourage the community to develop probabilistic models that are topology-aware for graph-structured tasks.

\section*{Acknowledgements and Disclosures}
YW acknowledges support from NSF CHI-1904822 and the Sloan Kettering Institute.
YW is a member of the Chodera Lab at Sloan Kettering Institute; a complete funding history for the Chodera lab can be found at \url{http://choderalab.org/funding}. 
YW is among the co-founders and equity holders of Uli, Inc. and Uli (Shenzhen) Techonology Co.\ Ltd.

\bibliographystyle{unsrtnat}
\bibliography{main}

\begin{thebibliography}{29}
\providecommand{\natexlab}[1]{#1}
\providecommand{\url}[1]{\texttt{#1}}
\expandafter\ifx\csname urlstyle\endcsname\relax
  \providecommand{\doi}[1]{doi: #1}\else
  \providecommand{\doi}{doi: \begingroup \urlstyle{rm}\Url}\fi

\bibitem[Kipf and Welling(2016)]{DBLP:journals/corr/KipfW16}
Thomas~N. Kipf and Max Welling.
\newblock Semi-supervised classification with graph convolutional networks.
\newblock \emph{CoRR}, abs/1609.02907, 2016.
\newblock URL \url{http://arxiv.org/abs/1609.02907}.

\bibitem[Xu et~al.(2018)Xu, Hu, Leskovec, and Jegelka]{xu2018powerful}
Keyulu Xu, Weihua Hu, Jure Leskovec, and Stefanie Jegelka.
\newblock How powerful are graph neural networks?
\newblock \emph{arXiv preprint arXiv:1810.00826}, 2018.

\bibitem[Gilmer et~al.(2017)Gilmer, Schoenholz, Riley, Vinyals, and
  Dahl]{gilmer2017neural}
Justin Gilmer, Samuel~S Schoenholz, Patrick~F Riley, Oriol Vinyals, and
  George~E Dahl.
\newblock Neural message passing for quantum chemistry.
\newblock \emph{arXiv preprint arXiv:1704.01212}, 2017.

\bibitem[Hamilton et~al.(2017)Hamilton, Ying, and
  Leskovec]{hamilton2017inductive}
Will Hamilton, Zhitao Ying, and Jure Leskovec.
\newblock Inductive representation learning on large graphs.
\newblock In \emph{Advances in neural information processing systems}, pages
  1024--1034, 2017.

\bibitem[Battaglia et~al.(2018)Battaglia, Hamrick, Bapst, Sanchez-Gonzalez,
  Zambaldi, Malinowski, Tacchetti, Raposo, Santoro, Faulkner,
  et~al.]{battaglia2018relational}
Peter~W Battaglia, Jessica~B Hamrick, Victor Bapst, Alvaro Sanchez-Gonzalez,
  Vinicius Zambaldi, Mateusz Malinowski, Andrea Tacchetti, David Raposo, Adam
  Santoro, Ryan Faulkner, et~al.
\newblock Relational inductive biases, deep learning, and graph networks.
\newblock \emph{arXiv preprint arXiv:1806.01261}, 2018.

\bibitem[Weisfeiler and Leman()]{weisfeiler1968reduction}
Boris Weisfeiler and Andrei Leman.
\newblock The reduction of a graph to canonical form and the algebra which
  appears therein.

\bibitem[Corso et~al.(2020)Corso, Cavalleri, Beaini, Liò, and
  Veličković]{corso2020principal}
Gabriele Corso, Luca Cavalleri, Dominique Beaini, Pietro Liò, and Petar
  Veličković.
\newblock Principal neighbourhood aggregation for graph nets, 2020.

\bibitem[Li et~al.(2018)Li, Han, and Wu]{DBLP:journals/corr/abs-1801-07606}
Qimai Li, Zhichao Han, and Xiao{-}Ming Wu.
\newblock Deeper insights into graph convolutional networks for semi-supervised
  learning.
\newblock \emph{CoRR}, abs/1801.07606, 2018.
\newblock URL \url{http://arxiv.org/abs/1801.07606}.

\bibitem[Oono and Suzuki(2019)]{DBLP:journals/corr/abs-1905-10947}
Kenta Oono and Taiji Suzuki.
\newblock On asymptotic behaviors of graph cnns from dynamical systems
  perspective.
\newblock \emph{CoRR}, abs/1905.10947, 2019.
\newblock URL \url{http://arxiv.org/abs/1905.10947}.

\bibitem[Srivastava et~al.(2014)Srivastava, Hinton, Krizhevsky, Sutskever, and
  Salakhutdinov]{JMLR:v15:srivastava14a}
Nitish Srivastava, Geoffrey Hinton, Alex Krizhevsky, Ilya Sutskever, and Ruslan
  Salakhutdinov.
\newblock Dropout: A simple way to prevent neural networks from overfitting.
\newblock \emph{Journal of Machine Learning Research}, 15\penalty0
  (56):\penalty0 1929--1958, 2014.
\newblock URL \url{http://jmlr.org/papers/v15/srivastava14a.html}.

\bibitem[Gal and Ghahramani(2016)]{gal2016dropout}
Yarin Gal and Zoubin Ghahramani.
\newblock Dropout as a bayesian approximation: Representing model uncertainty
  in deep learning, 2016.

\bibitem[Rong et~al.(2019)Rong, Huang, Xu, and
  Huang]{DBLP:journals/corr/abs-1907-10903}
Yu~Rong, Wenbing Huang, Tingyang Xu, and Junzhou Huang.
\newblock The truly deep graph convolutional networks for node classification.
\newblock \emph{CoRR}, abs/1907.10903, 2019.
\newblock URL \url{http://arxiv.org/abs/1907.10903}.

\bibitem[Wu et~al.(2019)Wu, Pan, Chen, Long, Zhang, and
  Yu]{DBLP:journals/corr/abs-1901-00596}
Zonghan Wu, Shirui Pan, Fengwen Chen, Guodong Long, Chengqi Zhang, and
  Philip~S. Yu.
\newblock A comprehensive survey on graph neural networks.
\newblock \emph{CoRR}, abs/1901.00596, 2019.
\newblock URL \url{http://arxiv.org/abs/1901.00596}.

\bibitem[Veličković et~al.(2018)Veličković, Cucurull, Casanova, Romero,
  Liò, and Bengio]{velickovic2018graph}
Petar Veličković, Guillem Cucurull, Arantxa Casanova, Adriana Romero, Pietro
  Liò, and Yoshua Bengio.
\newblock Graph attention networks, 2018.

\bibitem[Murphy et~al.(2018)Murphy, Srinivasan, Rao, and
  Ribeiro]{DBLP:journals/corr/abs-1811-01900}
Ryan~L. Murphy, Balasubramaniam Srinivasan, Vinayak~A. Rao, and Bruno Ribeiro.
\newblock Janossy pooling: Learning deep permutation-invariant functions for
  variable-size inputs.
\newblock \emph{CoRR}, abs/1811.01900, 2018.
\newblock URL \url{http://arxiv.org/abs/1811.01900}.

\bibitem[Neal(2012)]{neal2012bayesian}
Radford~M Neal.
\newblock \emph{Bayesian learning for neural networks}, volume 118.
\newblock Springer Science \& Business Media, 2012.

\bibitem[MacKay(1992)]{mackay1992practical}
David~JC MacKay.
\newblock A practical bayesian framework for backpropagation networks.
\newblock \emph{Neural computation}, 4\penalty0 (3):\penalty0 448--472, 1992.

\bibitem[Blei et~al.(2017)Blei, Kucukelbir, and McAuliffe]{Blei_2017}
David~M. Blei, Alp Kucukelbir, and Jon~D. McAuliffe.
\newblock Variational inference: A review for statisticians.
\newblock \emph{Journal of the American Statistical Association}, 112\penalty0
  (518):\penalty0 859–877, Apr 2017.
\newblock ISSN 1537-274X.
\newblock \doi{10.1080/01621459.2017.1285773}.
\newblock URL \url{http://dx.doi.org/10.1080/01621459.2017.1285773}.

\bibitem[Blundell et~al.(2015)Blundell, Cornebise, Kavukcuoglu, and
  Wierstra]{blundell2015weight}
Charles Blundell, Julien Cornebise, Koray Kavukcuoglu, and Daan Wierstra.
\newblock Weight uncertainty in neural networks, 2015.

\bibitem[Kingma et~al.(2015)Kingma, Salimans, and
  Welling]{kingma2015variational}
Diederik~P. Kingma, Tim Salimans, and Max Welling.
\newblock Variational dropout and the local reparameterization trick, 2015.

\bibitem[Zhang et~al.(2018)Zhang, Pal, Coates, and Üstebay]{zhang2018bayesian}
Yingxue Zhang, Soumyasundar Pal, Mark Coates, and Deniz Üstebay.
\newblock Bayesian graph convolutional neural networks for semi-supervised
  classification, 2018.

\bibitem[Chen et~al.(2018)Chen, Ma, and
  Xiao]{DBLP:journals/corr/abs-1801-10247}
Jie Chen, Tengfei Ma, and Cao Xiao.
\newblock Fastgcn: Fast learning with graph convolutional networks via
  importance sampling.
\newblock \emph{CoRR}, abs/1801.10247, 2018.
\newblock URL \url{http://arxiv.org/abs/1801.10247}.

\bibitem[Hasanzadeh et~al.(2020)Hasanzadeh, Hajiramezanali, Boluki, Zhou,
  Duffield, Narayanan, and Qian]{hasanzadeh2020bayesian}
Arman Hasanzadeh, Ehsan Hajiramezanali, Shahin Boluki, Mingyuan Zhou, Nick
  Duffield, Krishna Narayanan, and Xiaoning Qian.
\newblock Bayesian graph neural networks with adaptive connection sampling,
  2020.

\bibitem[Cai and Wang(2020)]{cai2020note}
Chen Cai and Yusu Wang.
\newblock A note on over-smoothing for graph neural networks, 2020.

\bibitem[Kingma and Ba(2017)]{kingma2017adam}
Diederik~P. Kingma and Jimmy Ba.
\newblock Adam: A method for stochastic optimization, 2017.

\bibitem[Delaney(2004)]{delaney2004esol}
John~S Delaney.
\newblock Esol: estimating aqueous solubility directly from molecular
  structure.
\newblock \emph{Journal of chemical information and computer sciences},
  44\penalty0 (3):\penalty0 1000--1005, 2004.

\bibitem[Mobley and Guthrie(2014)]{mobley2014freesolv}
David~L Mobley and J~Peter Guthrie.
\newblock Freesolv: a database of experimental and calculated hydration free
  energies, with input files.
\newblock \emph{Journal of computer-aided molecular design}, 28\penalty0
  (7):\penalty0 711--720, 2014.

\bibitem[Wang et~al.(2020)Wang, Zheng, Ye, Gan, Li, Song, Zhou, Ma, Yu, Gai,
  Xiao, He, Karypis, Li, and Zhang]{wang2020deep}
Minjie Wang, Da~Zheng, Zihao Ye, Quan Gan, Mufei Li, Xiang Song, Jinjing Zhou,
  Chao Ma, Lingfan Yu, Yu~Gai, Tianjun Xiao, Tong He, George Karypis, Jinyang
  Li, and Zheng Zhang.
\newblock Deep graph library: A graph-centric, highly-performant package for
  graph neural networks, 2020.

\bibitem[Paszke et~al.(2019)Paszke, Gross, Massa, Lerer, Bradbury, Chanan,
  Killeen, Lin, Gimelshein, Antiga, Desmaison, Kopf, Yang, DeVito, Raison,
  Tejani, Chilamkurthy, Steiner, Fang, Bai, and Chintala]{NEURIPS2019_9015}
Adam Paszke, Sam Gross, Francisco Massa, Adam Lerer, James Bradbury, Gregory
  Chanan, Trevor Killeen, Zeming Lin, Natalia Gimelshein, Luca Antiga, Alban
  Desmaison, Andreas Kopf, Edward Yang, Zachary DeVito, Martin Raison, Alykhan
  Tejani, Sasank Chilamkurthy, Benoit Steiner, Lu~Fang, Junjie Bai, and Soumith
  Chintala.
\newblock Pytorch: An imperative style, high-performance deep learning library.
\newblock In H.~Wallach, H.~Larochelle, A.~Beygelzimer, F.~d\textquotesingle
  Alch\'{e}-Buc, E.~Fox, and R.~Garnett, editors, \emph{Advances in Neural
  Information Processing Systems 32}, pages 8024--8035. Curran Associates,
  Inc., 2019.

\end{thebibliography}

\newpage
\onecolumn

\icmltitle{Stochastic Aggregation in Graph Neural Networks: Supplementary Material}




\begin{icmlauthorlist}
\icmlauthor{Yuanqing Wang}{mskcc,cornell,cuny}
\icmlauthor{Theofanis Karaletsos}{facebook}
\end{icmlauthorlist}

\icmlaffiliation{mskcc}{Computational and Systems Biology Program, Sloan Kettering Institute, Memorial Sloan Kettering Cancer Center, New York, NY 10065}

\icmlaffiliation{cornell}{Physiology, Biophysics, and System Biology Ph.D. Program, New York, NY 10065}

\icmlaffiliation{cuny}{M.F.A. Program in Creative Writing, City College of New York, City University of New York, New York, NY 10031}

\icmlaffiliation{facebook}{Facebook, Inc., Menlo Park, CA 94025}

\icmlcorrespondingauthor{Yuanqing Wang}{yuanqing.wang@choderalab.org}

\icmlkeywords{Graph Neural Network, Variational Inference}

\vskip 0.3in



\section{Theorems and Proofs}
\label{supsec-theorems}
\begin{theorem}
\label{one}
Only one aggregator $\rho$ is needed to discriminate between multisets $X$ with support $\mathbb{R}^C \setminus \{ \mathbf{0} \}$ after perturbation with some noise distribution $q$ on $\mathbb{R}^C$.
More formally, under some distribution $q$, $\rho(\xi_q(\mathbf{X}))$ and $\rho(\xi_q(\mathbf{Y}))$ are equal in distribution iff. there exist a permutation $P$ s.t. $[P\mathbf{X}]_i = [\mathbf{Y}_q]_i, \forall 1 \leq i \leq \mid \mathbf{X} \mid$.
\end{theorem}

We prove Theorem~\ref{one} for $\operatorname{SUM}$ aggregator on $\mathbb{R} \setminus \{0\}$, although is easy to expand to $\operatorname{MEAN}$ and $\operatorname{MAX}$ aggregators on $\mathbb{R}^C \setminus \{ \mathbf{0} \}.$
\begin{proof}[Proof]
Suppose we have two multisets $\mathbf{X} = \{x_i, i = 1, 2, ..., N_X\}$ and $\mathbf{Y} = \{y_j, i = 1, 2, ..., N_Y\}$.
We choose the multiplicative noise $Z \sim q = \operatorname{Uniform}(0, 1)$, with the moment generating function $M_Z(t) = \frac{e^{t}-1}{t}$, and $\operatorname{SUM}(\xi_q(X))$ and $\operatorname{SUM}(\xi_q(Y))$ are equal in distribution.
Thus, the moment generating function of $\operatorname{SUM}(\xi_q(X))$ is
\begin{equation}
M_{\operatorname{SUM}(\xi_q(\mathbf{X}))}(t) 
= M_{\sum x_i z_i}(t) 
= \frac{\prod (e^{x_i t} - 1)}{\prod x_i t}
\end{equation}
Since $\operatorname{SUM}(\xi_q(X))$ and $\operatorname{SUM}(\xi_q(Y))$ are equal in distribution, $M_{\operatorname{SUM}(\xi_q(X))}(t) = M_{\operatorname{SUM}(\xi_q(Y))}(t)$, and therefore $\frac{\prod (e^{x_i t} - 1)}{\prod x_i t} = \frac{\prod (e^{y_j t} - 1)}{\prod y_j t}$.
Considering the Taylor expansion of $\exp(\cdot)$, we have
\begin{equation}
\sum x_i^n = \sum y_j^n
\end{equation}
for any $n \in \mathbb{N^{+}}$.
Since $\forall x_i \neq 0$ and $\forall y_i \neq 0 $, we conclude that $\mathbf{X}$ and $\mathbf{Y}$ are equal.
\end{proof}

\begin{theorem}
\label{dirichlet-slow}
For any multiplicative noise distribution $q$ satisfying $\mid \mathbb{E}_{z \sim q}(z) \mid \geq 1$, any deterministic aggregator $\rho$, a node representation $\mathbf{X}$ of a graph, we have:
\begin{equation}
\mathbb{E}_{q}(\mathcal{E}(\rho(\xi_q(\mathbf{X})))) \geq \mathcal{E}(\rho (\mathbf{X}))
\end{equation}
\end{theorem}

\begin{proof}
We use $\mathcal{N}(\cdot)$ to denote the neighbor-finding operation.
By Jensen's inequality, we have:
\begin{align}
\begin{split}
&\mathbb{E}_{q}\big(\mathcal{E}(\rho(\xi_q(\mathbf{X})))\big) \\
&= \mathbb{E}_q\big(\frac{1}{2} \sum \mathbf{A}_{ij}(\frac{\rho(\mathcal{N}(v_i))}{\sqrt{1+d_i}} - \frac{\rho(\mathcal{N}(v_j))}{\sqrt{1+d_j}})^2 \big)\\
&\geq \frac{1}{2} \sum \mathbf{A}_{ij} \big(\frac{ \mathbb{E}_q(\rho(\mathcal{N}(v_i)))}{\sqrt{1+d_i}} - \frac{ \mathbb{E}_q(\rho(\mathcal{N}(v_j)))}{\sqrt{1+d_j}} \big)^2\\
&= \frac{1}{2} \sum \mathbf{A}_{ij}\big(\frac{\mathbb{E}_q(\rho(z_k u_k, u_k \in \mathcal{N}(v_i), z_k \sim q(z)))}{\sqrt{1+d_i}} - \frac{ \mathbb{E}_q(\rho( z_l u_l, u_l \in \mathcal{N}(v_j), z_l \sim q(z)))}{\sqrt{1+d_j}}\big)^2\\
&= \frac{1}{2} \sum \mathbf{A}_{ij}\mathbb{E}_q^2(z)\big(\frac{\rho(\mathcal{N}(v_i))}{\sqrt{1+d_i}} - \frac{ \rho(\mathcal{N}(v_j))}{\sqrt{1+d_j}} \big)^2\\
&\geq \frac{1}{2} \sum \mathbf{A}_{ij}\big(\frac{\rho(\mathcal{N}(v_i))}{\sqrt{1+d_i}} - \frac{ \rho(\mathcal{N}(v_j))}{\sqrt{1+d_j}}\big)^2 = \mathcal{E}(\rho (\mathbf{X}))
\end{split}
\end{align}
\end{proof}

\section{Extra Results}
\subsection{One sample across message-passing rounds vs. re-sampling.}
In Table~\ref{one-sample} we briefly study the performance of STAG between when we sample the noise distribution once during the forward pass rather than once per message-passing step.
Compared to Table 2, we observe an increase in the performance on Citeseer and a decrease on Cora.

\begin{table}[htbp]
    \centering
    \begin{tabular}{c c c c c}
        \hline
        & \multicolumn{2}{c}{Cora}
        & \multicolumn{2}{c}{Citeseer}\\
        \hline
        & 2 layers
        & 4 layers
        & 2 layers
        & 4 layers\\
        \hline
        $\operatorname{Normal}(1, 0.2)$
        & 80.04 ± 0.08 & 77.26 ± 0.75
        & 67.46 ± 0.70 & 60.52 ± 2.90
        \\
        
        $\operatorname{Normal}(1, 0.4)$
        & 80.14 ± 0.29 & 77.26 ± 0.75
        & 67.74 ± 0.46 & 61.52 ± 1.38
        \\
        
        $\operatorname{Normal}(1, 0.8)$
        & 80.68 ± 0.50 & 77.72 ± 1.20
        & 67.70 ± 1.23 & 61.58 ± 1.03
        \\
        \hline
        
        $\operatorname{Uniform}(0.8, 1.2)$
        & 79.64 ± 0.21 & 77.42 ± 1.62
        & 67.68 ± 0.41 & 61.10 ± 1.60
        \\
        
        $\operatorname{Uniform}(0.6, 1.4)$
        & 79.70 ± 0.29 & 77.00 ± 1.55
        & 67.80 ± 0.82 & 62.04 ± 1.00
        \\ 
        
        $\operatorname{Unifrom}(0.2, 1.8)$
        & 79.74 ± 0.48 & 77.50 ± 1.13
        & 67.26 ± 0.75 & 62.80 ± 1.60
        \\
        \hline
        
        $\operatorname{Bernoulli}(0.2)$
        & 79.90 ± 0.35 & 79.90 ± 0.35
        & 67.88 ± 0.78 & 62.16 ± 0.81
        \\
        
        $\operatorname{Bernoulli}(0.4)$
        & 80.56 ± 0.71 & 77.00 ± 0.78
        & 68.26 ± 1.25 & 61.28 ± 1.79
        \\
        
        $\operatorname{Bernoulli}(0.8)$
        & 15.70 ± 0.38 & 53.90 ± 0.87
        & 17.84 ± 0.52 & 19.00 ± 0.92
        \\
        \hline
        
    \end{tabular}
    \caption{Performance of STAG on citation graphs with same samples across rounds of message-passing.}
    \label{one-sample}
\end{table}

\subsection{Compatibility with other variants of GNN}
\label{supsec-var}
To illustrate that STAG is compatible with various types of GNNs, we test the performance of the non-adaptive version of STAG with GraphSAGE~\cite{hamilton2017inductive} and Graph Isomorphism Network (GIN)~\cite{xu2018powerful}.
The experiment setting in this section is identical to Table 2.
For GraphSAGE, we chose the $\operatorname{MEAN}$ function as the basic aggregator.
For GIN, the update function was chosen to be a single-layer neural network with ReLU activation; the basic aggregator was chosen to be $\operatorname{SUM}$.

As shown in Table~\ref{gin} and Table~\ref{sage}, STAG in general boosts the performance on both citation datasets.
STAG with continuous noise provides further performance improvement when used with GraphSAGE whereas Bernoulli noise enhances the test set accuracy further when used with GIN.

\begin{table}[htbp]
    \centering
    \begin{tabular}{c c c c c}
        \hline
        & \multicolumn{2}{c}{Cora}
        & \multicolumn{2}{c}{Citeseer}\\
        \hline
        & 2 layers
        & 4 layers
        & 2 layers
        & 4 layers\\
        \hline
        Deterministic
        & 79.20 ± 0.15 & 78.64 ± 1.63
        & 70.42 ± 0.27 & 63.66 ± 3.39
        \\
        \hline 
        
        $\operatorname{Normal}(1, 0.2)$
        & 79.46 ± 0.39 & 79.14 ± 0.83
        & 70.20 ± 0.35 & 66.10 ± 0.90
        \\
        
        $\operatorname{Normal}(1, 0.4)$
        & 79.34 ± 0.33 & 78.20 ± 1.19
        & 70.82 ± 0.50 & 63.38 ± 3.48
        \\
        
        $\operatorname{Normal}(1, 0.8)$
        & 79.26 ± 0.35 & 77.66 ± 1.07
        & 70.14 ± 0.61 & 65.68 ± 1.27
        \\
        \hline
        
        $\operatorname{Uniform}(0.8, 1.2)$
        & 79.24 ± 0.23 & 78.80 ± 1.30
        & 70.82 ± 0.60 & 65.54 ± 3.80
        \\
        
        $\operatorname{Uniform}(0.6, 1.4)$
        & 79.54 ± 0.15 & 79.06 ± 1.20
        & 69.78 ± 0.80 & 64.42 ± 2.79
        \\
        
        $\operatorname{Uniform}(0.2, 1.8)$
        & 79.64 ± 0.30 & 78.26 ± 1.40
        & 70.14 ± 0.58 & 65.34 ± 2.97
        \\
        \hline
        
        $\operatorname{Bernoulli}(0.2)$
        & 78.42 ± 0.20 & 77.90 ± 1.91
        & 70.22 ± 0.52 & 64.72 ± 1.04
        \\
        
        $\operatorname{Bernoulli}(0.4)$
        & 76.38 ± 0.52 & 76.94 ± 1.31
        & 70.02 ± 0.83 & 66.82 ± 1.26
        \\
        
        $\operatorname{Bernoulli}(0.8)$
        & 66.02 ± 0.32 & 65.86 ± 0.79
        & 62.00 ± 0.69 & 61.08 ± 1.78
        \\
        \hline

    \end{tabular}
    \caption{Performance of STAG with GraphSAGE~\cite{hamilton2017inductive}}
    \label{sage}
\end{table}

\begin{table}[htbp]
    \centering
    \begin{tabular}{c c c c c}
        \hline
        & \multicolumn{2}{c}{Cora}
        & \multicolumn{2}{c}{Citeseer}\\
        \hline
        & 2 layers
        & 4 layers
        & 2 layers
        & 4 layers\\
        \hline
        Deterministic
        & 74.78 ± 0.52 & 72.54 ± 0.28
        & 65.78 ± 0.45 & 61.60 ± 1.69
        \\
        \hline
        
        $\operatorname{Normal}(1, 0.2)$
        & 75.12 ± 1.01 & 73.52 ± 0.70
        & 66.94 ± 0.66 & 61.66 ± 0.93
        \\
        
        $\operatorname{Normal}(1, 0.4)$
        & 75.68 ± 0.49 & 73.32 ± 0.81
        & 66.74 ± 0.96 & 63.26 ± 0.58
        \\
        
        $\operatorname{Normal}(1, 0.8)$
        & 75.50 ± 0.84 & 74.48 ± 1.11
        & 66.90 ± 1.02 & 64.54 ± 1.35
        \\
        \hline
        
        $\operatorname{Uniform}(0.8, 1.2)$
        & 75.42 ± 0.97 & 73.18 ± 1.44
        & 66.30 ± 0.89 & 61.14 ± 1.20
        \\
        
        $\operatorname{Uniform}(0.6, 1.4)$
        & 75.66 ± 0.98 & 73.20 ± 1.62
        & 66.76 ± 0.90 & 62.86 ± 1.08
        \\
        
        $\operatorname{Uniform}(0.2, 1.8)$
        & 75.68 ± 1.13 & 74.76 ± 1.27
        & 66.76 ± 0.68 & 64.42 ± 0.79
        \\
        \hline
        
        $\operatorname{Bernoulli}(0.2)$
        & 76.30 ± 0.60 & 73.74 ± 0.70
        & 66.72 ± 0.84 & 64.08 ± 1.97
        \\
        
        $\operatorname{Bernoulli}(0.4)$
        & 76.82 ± 0.98 & 74.50 ± 1.28
        & 68.48 ± 0.68 & 64.20 ± 1.33 
        \\
        
        $\operatorname{Bernoulli}(0.8)$
        & 77.96 ± 0.76 & 77.00 ± 1.15
        & 69.06 ± 0.69 & 62.66 ± 1.44 
        \\
        \hline

    \end{tabular}
    \caption{Performance of STAG with GIN~\cite{xu2018powerful}}
    \label{gin}
\end{table}

\subsection{Stochastic Aggregation vs. Stochastic Weights}
\label{subsec-bbb}
\begin{table}[htbp]
    \centering
    \begin{tabular}{c c c  | c c }
        \hline
        & Cora & Citeseer & \# Params & Iter. Time\\
        \hline
        BBB & 79.28 ± 1.17 & 65.12 ± 1.93 & 368k & 48.6 ms\\
        \hline
    $\text{STAG}_\text{VI} (\mathbb{R}^{C})$ & 81.33 ± 0.62 & 68.53 ± 0.54 & 188k & 15.5 ms \\
   $\text{STAG}_\text{VI} (\mathbb{R}^{\mid \mathcal{E} \mid \times C})$ & 81.38 ± 0.40 & 71.28 ± 0.65 & 1186k & 30.0 ms\\
    $\text{STAG}_\text{MLE}$ (best) & 80.34 ± 0.45 & 69.22 ± 0.87 & 184k & 9.3 ms\\
    \hline
    \end{tabular}
    \caption{Performance of Variational Inference on Weight Space (BBB)~\cite{blundell2015weight}}
    \label{bbb}
\end{table}

We compare our $\operatorname{STAG}_\text{VI}$ framework with Bayes-by-Backprop (BBB)~\cite{blundell2015weight} which performs variational inference over the weight posterior of the GNN using a factorized Normal distribution over weight space, corresponding to a \textit{mean-field} assumption.
\begin{equation}
Q(\mathbf{W}) = \prod Q(\mathbf{W}[i, :])
\end{equation}
with $Q(\mathbf{W}) = \mathcal{N}(\mu_\mathbf{W}, \sigma_\mathbf{W})$.

We use a similar experimental setting (two layer GCN, 128 units each, ReLU activation function, Adam optimizer with $10^{-3}$ learning rate) except that we infer weight posteriors and report the VI performance as well as efficiency in Table~\ref{bbb}.
We observe that STAG has a higher performance than BBB on test set.
It is also faster to train and has a better parameter efficiency.

\section{Experiment Details}
\label{supsec-detail}
\subsection{Implementation Details}
The models and the scripts necessary for all the experiments are implemented in Python 3.6 with Deep Graph Library (DGL)~\cite{wang2020deep} and PyTorch~\cite{NEURIPS2019_9015}.
We release the code with MIT open-source license here: \url{https://github.com/yuanqing-wang/stag.git}.

\subsection{Computational Infrastructure}
The experiments are carried out on a single NVIDIA Tesla V100\textsuperscript{\tiny\textregistered} GPU with 32 GB memory.
The speed benchmark experiments are done with two-layer GCN on Cora dataset.

\subsection{Datasets}
The citation datasets, Cora and Citeseer, contain one graph each.
We split the citation datasets in the same process as \citet{DBLP:journals/corr/KipfW16}---140 training nodes, 500 validation nodes, and 1000 test nodes for Cora and 120 training nodes, 500 validation nodes, and 1000 test nodes for Citeseer.

When it comes to molecule datasets, ESOL~\cite{delaney2004esol} is a dataset containing water solubility data (log mol per liter) of 1128 organic small molecule.
FreeSolv~\cite{mobley2014freesolv} provides experimental hydration free energy (kcal/mol) for 642 small molecules in water.
We used the atom featurization provided by DGLLife~\cite{wang2020deep} and randomly (with fixed random seed) split the both molecule datasets into training/validation/test sets (80:10:10).

\subsection{Numerical Optimization}
Poisson negative log likelihood loss function is used for citation graph node classification (Cora and Citeseer); mean squared error (MSE) loss is used for molecule graph regression (ESOL~\cite{delaney2004esol} and FreeSolv~\cite{mobley2014freesolv}).

We used Adam~\cite{kingma2017adam} optimizer for all experiments.
For citation graph benchmark results in Table 2 for STAG, DropEdge (DE)~\cite{DBLP:journals/corr/abs-1907-10903}and Graph DropConnect (GDC)~\cite{hasanzadeh2020bayesian}, we followed the experimental setting from \citet{hasanzadeh2020bayesian} and used a learning rate of $5 * 10^{-4}$ with a L2 regularization factor of $5 * 10^{-4}$ for the first layer.
For molecule graph benchmark results, we used a learning rate of $10^{-3}$ without regularization.
For all variational inference benchmark experiments and the toy example for multiset expressiveness (Figure 3), we used a learning rate of $10^{-3}$.
For performance deterioration for deep GNN toy experiment (Figure 5), we followed the experiment setting from \citet{DBLP:journals/corr/KipfW16} and used a learning rate of $10^{-2}$ and L2 regularization with factor $5 * 10^{-4}$.

We used a $\operatorname{SUM}$ function followed by a two-layer neural network with 128 units each layer and ReLU activation function to pool the node representation into graph representation in molecular graph regression tasks.

In both benchmark experiments, we trained the model for 2000 epochs with early stopping using the validation set.
For the performance deterioration with depth experiment (Figure 5), we trained the model for 400 epochs without early-stopping.

\subsection{Hyperparameters}
We used Graph Convolutional Network (GCN)~\cite{DBLP:journals/corr/KipfW16} as the GNN layers throughout the work.
Except in Figure 5 where we followed the setting in  \citet{DBLP:journals/corr/KipfW16} and used 16 units for that layer, we used 128 units everywhere.
ReLU activation functions are used everywhere.

For variational inference benchmark, since two-layer model always outperform four-layer models in a non-adaptive setting, we employed two-layer models only. 
Using the validation set, we tuned the initial values of the variational posterior $\mu_0, \sigma_0$ parameters as well as the standard deviation of the prior on the edge weights $\sigma_\text{prior}$;
the mean of the prior is fixed to be $1.0$.
We report the hyperparameters used to produce the VI results in Table~\ref{tab:hyperparameters}.
For transductive models, the initial values tuned in Table Table~\ref{tab:hyperparameters} are used as initial bias of the last feed-forward layer to predict $\mu$ and $\sigma$;
the weight of this last layer is initialized from $\mathcal{N}(0, 0.01)$ for layers predicting $\mu$ and $\mathcal{N}(0, 0.001)$ for layers predicting $\sigma$;
the rest of the initialization are set to default (Kaiming uniform).
We adopted the hyperparameters for the Beta-Bernoulli Graph DropConnect (BBGDC) model from its original publication.~\cite{hasanzadeh2020bayesian}.
In Figure~\ref{depth}, we used a dropout rate of 0.5 as is used in \citet{DBLP:journals/corr/KipfW16} and the basic noise distribution for STAG is $\mathcal{N}(1.0, 1.0)$. 
In Section~\ref{bbb}, we tuned the initialization values and the priors on the weight distributions.

During training, only one sample is acquired to estimate the gradient;
during inference, 32 samples are used everywhere to form the prediction.

\begin{table}[h]
    \centering
    \begin{tabular}{c c c c | c c c |  c c c | c c c }
    \hline
    & \multicolumn{3}{c}{Cora}
    & \multicolumn{3}{c}{Citeseer}
    & \multicolumn{3}{c}{ESOL}
    & \multicolumn{3}{c}{FreeSolv}
    \\
    \hline
    
    & $ \mu_0$ & $\log \sigma_0$ & $\sigma_\text{prior}$
     & $ \mu_0$ & $\log \sigma_0$ & $\sigma_\text{prior}$
      & $ \mu_0$ & $\log \sigma_0$ & $\sigma_\text{prior}$
       & $ \mu_0$ & $\log \sigma_0$ & $\sigma_\text{prior}$
    \\
    \hline
    
    $\text{STAG}_\text{VI} (\mathbb{R})$
    & 0.5 & 1.0 & 0.2
    & 0.5 & 0.0 & 0.5
    & 0.5 & -1.0 & 0.1
    & 0.1 & -1.0 & 1.0
    \\
    
    $\text{STAG}_\text{VI} (\mathbb{R}^C) $
    & 0.25 & 2.0 & 1.0
    & 0.25 & 2.0 & 0.5
    & 1.0 & 0.0 & 0.5
    & 0.1 & 0.0 & 0.5
    \\
    
    $\text{STAG}_\text{VI} (\mathbb{R}^{\mid \mathcal{E} \mid}) $
    & 0.5 & 1.5 & 0.5
    & 0.5 & 1.5 & 0.5
    & 0.1 & 0.0 & 1.0
    & 0.5 & -2.0 & 1.0
    \\
    
    $\text{STAG}_\text{VI} (\mathbb{R}^{\mid \mathcal{E} \mid \times C})$
    & 0.5 & 1.0 & 0.5
    & 0.5 & 1.0 & 1.0
    & 0.1 & -1.0 & 0.1
    & 1.0 & 0.0 & 0.1
    \\
    \hline
    \end{tabular}
    \caption{Hyperparameters used in VI experiments.}
    \label{tab:hyperparameters}
\end{table}
\end{document}